\documentclass[sigconf]{aamas}

\usepackage{balance} 

\usepackage{amsmath}
\usepackage{amsthm}

\usepackage[ruled, vlined, linesnumbered]{algorithm2e}
\usepackage{caption}
\usepackage{subcaption}

\setcopyright{ifaamas}
\acmConference[AAMAS '22]{Proc.\@ of the 21st International Conference
on Autonomous Agents and Multiagent Systems (AAMAS 2022)}{May 9--13, 2022}
{Auckland, New Zealand}{P.~Faliszewski, V.~Mascardi, C.~Pelachaud,
M.E.~Taylor (eds.)}
\copyrightyear{2022}
\acmYear{2022}
\acmDOI{}
\acmPrice{}
\acmISBN{}

\title{Learning Cooperation and Online Planning Through Simulation and Graph Convolutional Network}

\author{Rafid Ameer Mahmud}
\affiliation{
    \institution{University of Dhaka}}
\email{rafidameer-2016814404@cs.du.ac.bd}

\author{Fahim Faisal}
\affiliation{
    \institution{University of Dhaka}}
\email{fahim-2016714432@cs.du.ac.bd}

\author{Saaduddin Mahmud}
\affiliation{
  \institution{University of Massachusetts Amherst}}
\email{smahmud@umass.edu}

\author{Md. Mosaddek Khan}
\affiliation{
    \institution{University of Dhaka}}
\email{mosaddek@du.ac.bd}

\begin{abstract}
Multi-agent Markov Decision Process (MMDP) has been an effective way of modelling sequential decision making algorithms for multi-agent cooperative environments. A number of algorithms based on centralized and decentralized planning have been developed in this domain. However, dynamically changing environment, coupled with exponential size of the state and joint action space, make it difficult for these algorithms to provide both efficiency and scalability. Recently, Centralized planning algorithm FV-MCTS-MP and decentralized planning algorithm \textit{Alternate maximization with Behavioural Cloning} (ABC) have achieved notable performance in solving MMDPs. 
However, they are not capable of adapting to dynamically changing environments and accounting for the lack of communication among agents, respectively. Against this background, we introduce a simulation based online planning algorithm, that we call SiCLOP, for multi-agent cooperative environments. Specifically, SiCLOP tailors Monte Carlo Tree Search (MCTS) and uses Coordination Graph (CG) and Graph Neural Network (GCN) to learn cooperation and provides real time solution of a MMDP problem. It also improves scalability through an effective pruning of action space. Additionally, unlike FV-MCTS-MP and ABC, SiCLOP supports transfer learning, which enables learned agents to operate in different environments. We also provide theoretical discussion about the convergence property of our algorithm within the context of multi-agent settings. Finally, our extensive empirical results show that SiCLOP significantly outperforms the state-of-the-art online planning algorithms.  

\end{abstract}


\newcommand{\BibTeX}{\rm B\kern-.05em{\sc i\kern-.025em b}\kern-.08em\TeX}

\begin{document}


\pagestyle{fancy}
\fancyhead{}


\maketitle 


\section{Introduction}

    Sequential decision making models for multi-agent environments hold the key to many real life problems such as autonomous vehicles

    \cite{shalev2016safe}, controlling robots \cite{kober2013reinforcement-controllingrobots}, resource allocation \cite{yang2018mean-resourceallocation}, games with multiple types of agents such as Starcraft II \cite{rashid2018qmix-starcraft} and so on. Cooperation among agents is an essential part of these problems. The agents need to learn collaboration in order to work together towards a common goal. This adds to the difficulty of making decisions because the agents must interact not only with the environment, but also with one another. As a result, they need to adapt to the policies of other agents to learn cooperation. Another fundamental challenge of this domain is the curse of dimensionality. 
   With the increase in the number of agents, the joint action space and number of states grows exponentially, rendering single agent planning algorithms inefficient in these cases.
    In multi-agent systems, the sequential decision making problem can be modeled as a variant of Markov Decision Process (MDP) \cite{sutton2011reinforcement-sutton}, called Multi-agent Markov Decision Process (MMDP) \cite{boutilier1996planning-mmdp}. \par

    Over the years, a number of algorithms have been proposed to solve MMDPs. They are broadly divided into two categories. Centralized planning and decentralized planning. To solve MMDPs through centralized planning, the naive approach is to extend the idea of single agent algorithms by treating all the agents as a single agent and combining their joint action space to represent the action space. However, due to the curse of dimensionality, this naive approach fails to scale up to large environments. Therefore, a different approach, decentralized planning, has been widely used to solve MMDPs. This is accomplished by decomposing the multi-agent characteristics through decentralizing their value function. Best et al. \cite{best2019dec} proposed a decentralized online planning algorithm to solve the dimensionality problem by using parallel MCTS trees and periodic communication. On the other hand, Aleksander et. al. \cite{czechowski2020decentralized} introduced an algorithm ABC, where agents do not use communication, rather they train their individual policy prediction functions one at a time and try to induce cooperation by using the learned agent models. However, ABC experiences poor performance since, because of the lack of communication, it is unable to account for penalties generated by agent interaction (see Section 4).

    
    
    Guestrin et al. \cite{guestrin2001multiagent} showed that communication among agents can be represented as a coordination graph (CG) and joint value functions can be estimated from the higher order value factorization. This encouraged other approaches that use coordination graph to learn value functions or for policy generation. Choudhury et al. \cite{choudhury2021scalable} proposed simulation based \textit{anytime online} planning algorithm FV-MCTS-MP, where factored value function is used with CG to incorporate communication among agents. The proposed solution solves both the cooperation and dimensionality problem but is not applicable in dynamically changing environments. In such cases, the state space grows exponentially as the size of the environment rises, limiting FV-MCTS-MP's ability to calculate state values and deliver good solutions given limited computational resources.

      \par
    
    Against this background, we introduce a MCTS based multi-agent anytime planning algorithm, SiCLOP, that can scale for large environments and provide realtime action selection capabilities. SiCLOP uses Graph Convolutional Network (GCN) \cite{kipf2016semi} for individual agent policy prediction and propagates the updated policies to other agents through  Coordination Graphs (CG). The agents use only their local information for policy prediction and training. This enables them to learn cooperation in a small environments and transfer that knowledge to a larger setting, making the training easier and faster. We also introduce an effective pruning method to reduce the size of the joint action space by sampling only a small subset of all possible joint actions using Gibbs sampling \cite{geman1984stochastic} and iteratively updating the actions of every individual agent. Finally, we evaluate our algorithm empirically, and the result depicts that our algorithm outperforms the state-of-the-art online planning algorithms in terms of solution quality, adaptability, robustness and scalability.
    
  \par

    

\section{Problem Formulation and Background}

In this section, we first formulate the MMDP problem. We then review the general ideas of solving MMDPs with $online planning$. Afterwards, we discuss the impact of Coordination Graph (CG) and graph based neural networks in improving the solution quality and adaptability, 

	\subsection*{Multi-agent Markov Decision Process, MMDP}
	    An MMDP can be formalized as a tuple $(\mathcal{N},S,A,\tau,\mathcal{U})$. $\mathcal{N}$ is the set of all agents, indexed by $1,....,n$. We use $-i$ to refer to all other agents except $i$. $S$ is the set of all possible states and $A = A_{1} \times ... \times A_{n}$, is a set of action spaces where $A_{i}$ is the set of actions for agent $i \in \mathcal{N}$. A joint action $a \in A$ is the combination of individual actions of all agents, $a = \langle a_{1},...,a_{n} \rangle , a_{i}\in A_{i}$. $\tau : S \times A \rightarrow S$ is a transition function that takes a joint action and changes the state according to the deployed environment specifications. $\mathcal{U} = (u_{i},...,u_{n})$ is the global utility function and for every agent $i \in \mathcal{N}, u_{i}:  S \rightarrow \mathbb{R}$ denotes the utility function to determine the payoff of states in $S$ for that agent.
        A \textit{policy} is the task of deciding which action to choose in any state. The policy for the $i$-th agent is denoted by $\pi_{i}: S\rightarrow|A_{i}|$, a distribution over the action space depending on the state the agent is in. The collection of policies of all the agents is called the joint policy $\pi = \langle \pi_{1},...,\pi_{n} \rangle$. If the joint policy is fixed, then all the agents take actions according to that, and the expected reward or the  \textit{state value} $V_{i}^{\pi}(s)$ of an state $s$ for agent $i$ is:
        \begin{equation}
        \label{state_value}
            V_{i}^{\pi}(s) = u_{i}(s) + \sum_{a \in A} \pi(a|s) V_{i}^{\pi}(\tau(s,a))
        \end{equation}
        And the expected reward of an action or \textit{action value} $Q_{i}^{\pi}(s,a)$ of state $s$ for agent $i$ in joint action $a$ is:
        \begin{equation}
            Q_{i}^{\pi}(s,a) = u_{i}(s) + V_{i}^{\pi}(\tau(s,a))
        \end{equation}
    \paragraph{Best Response Policy} \hspace{-1mm} This $state \ value$ assumes that the policies of other agents are stationary, although in multi-agent policy learning contexts, this is rarely the case. Even if agent $i$ does not change its policy, the expected reward will change when other agents change their policies. \par
    In a multi-agent setting, the agents have to search the same state space to create strategies that maximise their payoff. Here, the goal of the agents is to adapt to the non-stationary policies of other agents. A $best \ response$ policy for agent $i$ is defined by the set $BR(\pi_{-i})$. It is the set of policies that maximise the reward assuming that the joint policy of other agents $\pi_{-i}$ is stationary.A policy $\pi^{*}$ belongs to  $BR(\pi_{-i})$ if and only if the following inequality holds:
    \begin{equation}
        \forall \pi_{i};\forall s\in S;\pi_{i}^{*} \in BR_{i}(\pi_{-i})  \quad V_{i}^{(\pi_{i}^{*}, \pi_{-i})}(s) \geq V_{i}^{(\pi_{i}, \pi_{-i})}(s)
    \end{equation}
    
    \subsection*{Online Planning}
    Simulation based methods provide a means of policy learning and planning in non-stationary environments where it is possible for the agents to find their best response policy according to their current model of other agents. Monte Carlo Tree Search (MCTS) \cite{kocsis2006bandit} is an effective tree searching algorithm that uses simulation data to balance between exploration of new solutions and exploitation of solutions with potential. There are four major steps in the traditional MCTS simulation process. At any given moment, with an already expanded simulation tree, the first step is to start from the root and select a new leaf by going down the tree. While traversing, at every node, scores for its child nodes are calculated and the child with the best score is selected to go further down until a leaf is reached. The score is calculated with,
    \begin{equation}
        score_{a} = Q(s,a) + c*\sqrt{(\log N_{i}) / n_{i}}
        \label{eq:PUCT}
    \end{equation}
    where $Q(s,a)$ is the action value, $n_{i}$ is the visit count of action $a$ and $N_{i}$ is the total visit count of state $s$. The constant $c$ determines the tradeoff between exploration and exploitation. The next step is to expand the selected leaf and add its child nodes to the tree, called the expansion step. After selection and expansion, the state value of the child nodes are estimated, traditionally with random rollouts. The last step is to carry the newfound information about the expanded node to its ancestors through backpropagation. In an simulated tree, every node in that tree holds information about its expected reward, the number of times that node was visited and the edges to its child nodes.\par
    AlphaZero \cite{silver2017mastering} takes a different approach with MCTS. In recent studies, it has been shown that generating policies for large state and action spaces can be done with deep neural networks rather than random plays \cite{mnih2015human}. Hence, instead of estimating state values with computationally exhausting rollouts, a neural network is used in AlphaZero. The neural network also provides an initial policy for the action space to better guide the tree searching process. This not only makes the process faster, but also helps focus on better solutions.
    
    \subsection*{Coordination Graph and Graph Neural Networks}
    \begin{figure}
        \centering
        \includegraphics[width=0.4\textwidth]{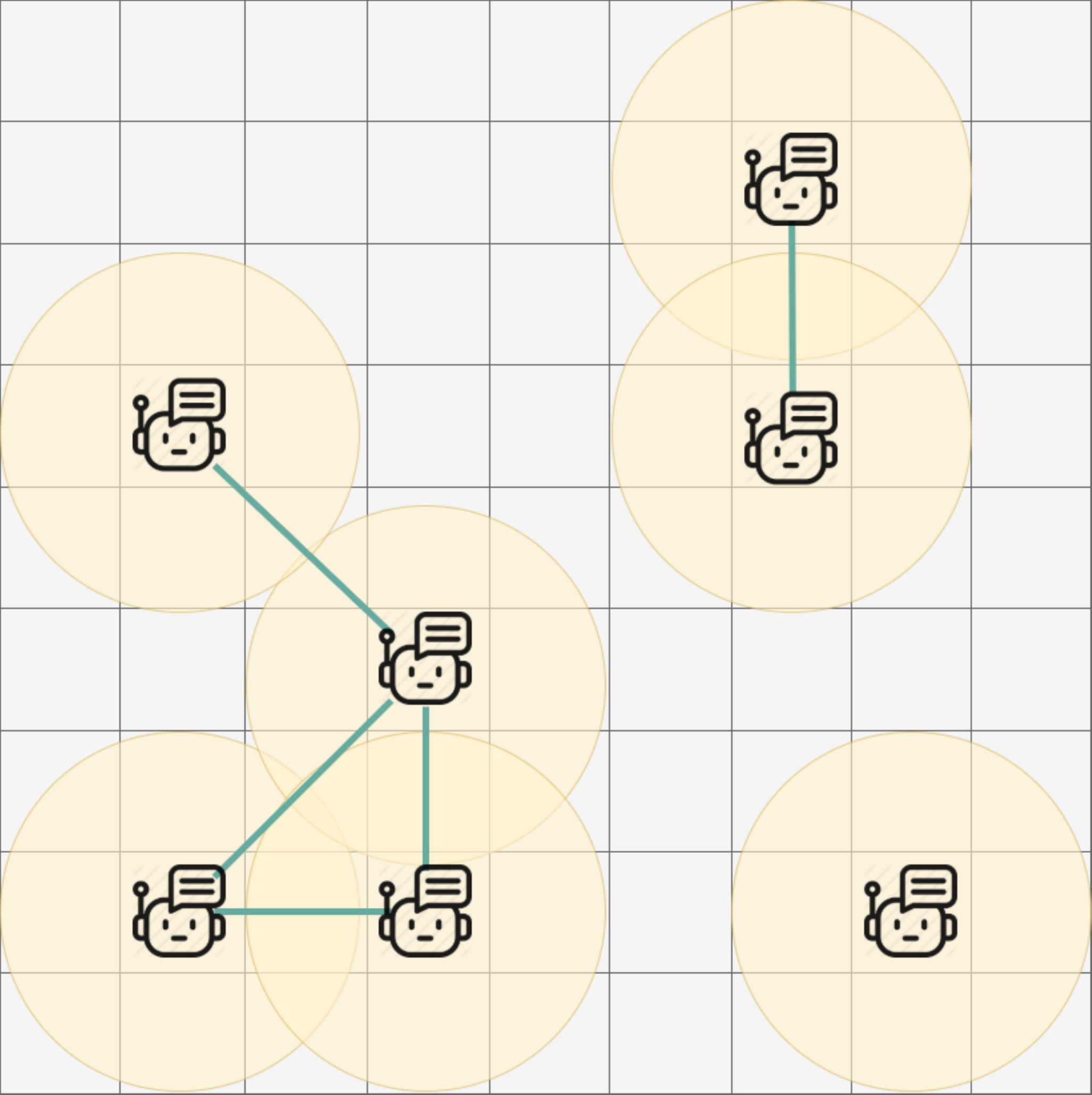}
        \caption{Coordination Graph. The agents have fixed ranges and they are connected if their communication ranges intersect. The edges determine which agents share information.}
        \label{fig:cg}
    \end{figure}
    
    In multi-agent environments, every agent interacts with a subset of other agents. These interactions can be modeled using a graph structure called Coordination Graph (CG), where the nodes are the agents and interacting agents have edges between them. This structure captures the locality of the dependencies of the agents regarding payoffs. This locality of dependency also determines the cooperative policies of the agents. In Figure \ref{fig:cg}, the agents create a forest of connected graphs based on their range. The edge creation criteria can be predefined and the generated CG will denote which agents need to communicate.\par
    
    New models and architectures have been brought to light that use graphs in training models to complete tasks regarding graphical structure, called Graph Neural Networks (GNN). The use of GNNs has been shown to be effective in agent modeling \cite{bohmer2020deep}\cite{liu2020multi}. Graph Convolutional Networks (GCN) is a popular GNN for graph embedding. GCN uses a message passing mechanism that lets the nodes of a graph to aggregate and retain information about their neighbors of different distances. This greatly helps in representation learning of graphical structures and use it to generate practical solutions.


\section{Multi-Agent Online Planning}


    \begin{algorithm}
    \caption{SiCLOP}
    \label{alg:pipeline}
    \KwIn{Set of initial states $S$, set of agents $\mathcal{N}$, limited steps $t$, data storage $d$, transition function $\tau$, utilility dunction $\mathcal{U}$}
    \SetKwBlock{Begin}{function}{end function}
    \ForAll{$state \in S$}
    {
        \For{$step$ in $0:t$}
        {
            $a \gets$ SiCLOP-S ($state,\mathcal{N}$)\\
            $state \gets \tau(state,a)$ \\
            \If{$state$ is terminal}
            {
                continue
            }
        }
        $reward_{i} \gets \mathcal{U}(state,a)$ \\
        \For{$step$ in $0:t$}
        {
            store $(state_{step},reward)$ in data storage $d$
        }
        update \texttt{SiCLOP-P\textsubscript{NN}} neural network with Algorithm \ref{alg:train}
    }
    \Begin($\texttt{SiCLOP-S} {(}s,\mathcal{N}{)}$)
    {
        initialize $root$ with state $s$ \\
        \While{time limit not exceeded}
        {
            $\Delta \gets$ select leaf using Eq. \ref{eq:PUCT} \\
            \If{$\Delta$ is terminal}
            {
                break
            }
            $\Lambda \gets \texttt{SiCLOP-P}(\Delta_{s},\mathcal{N},k)$ \\
            \For{$a$ in $\Lambda$}
            {
                create child node of $\Delta$ with state $\tau(\Delta_{s},a)$ \\
                $\rho \gets$ state value predicted by SiCLOP-P\textsubscript{NN} \\
                $reward \gets \mathcal{U}(\Delta_{s},a) + \rho$ \\
                propagate $reward$ to ancestors up to $root$
            }
        }
        $a \gets$ most visited joint action of $\mathcal{N}^\prime$ in $root$ \\
        \textbf{return} $a$
    }
    \Begin($\texttt{SiCLOP-P} {(}s,\mathcal{N},k{)}$)
    {
        initialize set of joint actions $\Lambda = \{\}$ \\
        $\pi\gets$ SiCLOP-P\textsubscript{NN}($preprocess(s)$) \\
        $a \gets$ sample joint action from $\pi$ \\
        \While{$|\Lambda|<k$}
        {
            \For{$i$ in $1:|\mathcal{N}|$}
            {
                $s^\prime \gets \tau(s, a_{-i})$ \\
                $\pi_{i} \gets$ SiCLOP-NN($preprocess(s^\prime)$)[i] \\
                $a_{i} \gets$ sample joint action from $\pi_{i}$ \\
            }
            $\Lambda \gets \Lambda \cup a$ \\
        }
        \textbf{return} $\Lambda$
    }
    \Begin($\texttt{preprocess} {(}s{)}$)
    {
        initialize $s^{*}$, empty set of observation states \\
        \For{$i$ in $1:|\mathcal{N}|$}
        {
            $o_{i} \gets$ observed elements of agent $i$ within fixed range concatenated with environmental details\\
            $s^{*} \gets s^{*} \cup o_{i}$
        }
        $CG \gets$ coordination graph based on fixed rules \\
        \textbf{return} $s^{*}, CG$
    }
    \end{algorithm}
    
    Our proposed algorithm  \textit{\textbf{Si}mulation based \textbf{C}ooperation \textbf{L}earner and \textbf{O}nline \textbf{P}lanner} (SiCLOP) is an MCTS based MMDP solver. SiCLOP has three major components, simulator SiCLOP-S, pruner SiCLOP-P, and predictor SiCLOP-NN. SiCLOP-S modifies the traditional MCTS to accommodate the pruning capabilities of SiCLOP-P and generate a deep simulation tree to find optimal policies. SiCLOP-P is a novel pruning technique that iterates over the joint policy space of the agents and uses best response policy predictor SiCLOP-NN to sample a small number of potential joint actions from the exponential joint action space. Lastly, SiCLOP-NN is a GCN based neural network that takes the local information and interaction of the agents (represented by CG) as input to predict best response policies for the agents and estimate the state value. Combining these three components, SiCLOP becomes a scalable and adaptive MMDP solver that is also capable of transfer learning. \par
    The MMDP problems are represented as initial states. The algorithms takes a set of initial states $S$, set of agents $\mathcal{N}$ and step limit $t$. SiCLOP is run for every initial state $s$ in $S$ in lines 1-10. For $t$ steps, the simulator SiCLOP-S finds an optimum joint action in line 3, and the state transitions to the next step in line 4, and repeats unless terminated in line 5. After $t$ steps, the simulation for that initial state $s$ is finished and the rewards are calculated in line 7. Then in lines 8-9, the \textit{state-action} and \textit{state-reward} pairs are stored in a database $d$ and then the database is used to train SiCLOP-NN in line 10. To summarize this section, the goal is to use simulation (depends on the predictor) to find solutions and then use data from that simulation  to improve the predictor so that next time, SiCLOP can find even better solutions.
    
    
    
    \subsection{SiCLOP-S, Simulation and Action Selection}
    
        The simulator starts with parameters initial state $s$ and set of all the agents $\mathcal{N}$. The $root$ of the simulation tree is initialized by the initial state $s$ in line 12. \par
        There are four steps in the simulation process and they are repeated multiple times within a fixed amount of time in the \textbf{while} loop in lines 13-22. In the first step, a leaf is reached from the $root$ in line 14 by frequently calculating the scores of the children of a node with Eq. \ref{eq:PUCT} and choosing the child with the highest score. Visit counts and state values of visited nodes, while traversing down the tree, are updated. After choosing a leaf node, its child nodes are created and the leaf is labeled as expanded. The number of possible child nodes is exponential in magnitude with respect to the number of agents. So, in the second step, a small subset of joint actions are sampled from the joint action space with SiCLOP-P in line 17 and the child nodes are created from the set of sampled actions to complete step three and four in lines 18-22. After creating a child node in line 18, its state value is determined in lines 20-21 by calculating penalties and prediction made by SiCLOP-NN. Then the state value is backpropagated to its ancestors in line 22. \par
        At the end of generating the simulation tree, the most visited joint action which denotes the joint action with the most potential, is selected in line 23 and then returned.
    
    \subsection{SiCLOP-P, Pruning the Joint Action Space}
        The SiCLOP-P component of the algorithm described in lines 25-35, focuses on pruning the joint action space for SiCLOP-S. We propose a novel joint action sampling method inspired by \textit{Gibbs sampling} to effectively search over the joint policy space. First, we elaborate the sampling process. The predictor SiCLOP-NN generates the best response policy for the agents given a state. Using this predictor, the best response policy of every agent is sampled considering the agent policies as static periodically. The sampling runs for multiple cycle. At cycle $k$, the policy of agent $i$ is updated with the conditional best response policy
        \begin{equation}
            \pi^{k}_{i} = \Psi(\pi_{i}|\pi^{k}_{1},...,\pi^{k}_{i-1},\pi^{k-1}_{i+1},...,\pi^{k-1}_{n})
        \end{equation}
        In SiCLOP-P, the initial policy is represented by initial joint action selection in line 28. Then, in lines 29-34, a fixed number of joint actions are sampled. The agents are selected periodically in lines 30-32 where the current policy of agent $i$ is replaced by the best response policy. To update the policy, all other agents' policies are considered to be static. The current adopted policies of the agents are represented by $s^\prime$ in line 31. Then the best response policy is updated in line 32 and new action is sampled from that policy in line 33. At the end of the cycles, in line 34, the updated joint action is added to the sampled joint actions and after a fixed number of cycles, the set of sampled joint actions are returned. \par
        
        The policy generator SiCLOP-P\textsubscript{NN} predicts the best response policy for an agent. So, the generated policy for agent $i$ belongs to $BR_{i}(\pi_{-i})$. In Lemma~\ref{lemma:policy}, we show that SiCLOP-P eventually converges to an equilibrium joint policy for a group of agents.
        
        \begin{lemma}
        \label{lemma:policy}
        In any iteration of sampling for agent $i$ in group $\mathcal{N}$, the new sampled policy $\pi^{*}_{i} \in BR_{i}(\pi_{-i})$ will lead to non-decreasing reward and iterating through all the agents of $\mathcal{N}^\prime$ will lead to an equilibrium for that group.
        \end{lemma}
        \begin{proof}
        If the current joint policy of all the agents is $\pi$ and the updated policy for agent $i$ is $\pi_{i}^{*}$, then for all $s \in S:$
        \begin{align}
            \forall i \in \mathcal{N}^\prime \quad V_{i}^{(\pi_{i}^{*}, \pi_{-i})}(s) \geq V_{i}^{(\pi_{i}, \pi_{-i})}(s)
        \end{align}
        If all the agents in the group acquire non-decreasing rewards, the combined reward will also be non-decreasing. As both the joint action space and reward are finite, iterating through the policy space will converge to a joint policy where no individual policy can be updated to get a better result, hence converging to an equilibrium for the group $\mathcal{N}^\prime$.
        \end{proof}
        This shows that with a well trained policy prediction function, the joint policy of every group will converge to a best response policy equilibrium. But in non-stationary environments, there might not a single joint policy equilibrium for the agents to converge to. In that case, the agents continuously create new best response policies to maximise their respective payoffs.\par
        
        The characteristics of the sampled joint actions change gradually throughout the learning process. In the early stages of learning, the policies of the agents are initialized to be random or according to prior knowledge. The set sampled joint actions  represent that initial joint policy. If it is random, then the agents explore different strategies through simulation as a group and develop their policies by finding strategies that lead to better payoffs. At the same time, agents adapt to other agents' policies and search for best responses. And if it is initialized with handcrafted policy, then the agents evolve by finding counter strategies and updating their existing policies. In both cases, they initially search randomly for new strategies and select the optimal ones in simulation. After repeating the learning process, the agents start to generate more definitive policies, meaning some actions are more preferred than others in their respective policies. This causes the sampling method to sample joint actions that have higher probability of being chosen to maximize rewards for all the agents. Directly predicting the optimal policy can concentrate on single policies and might lose robustness. Rather, simulating a search tree tests different ways of solving a MMDP with a set of best response policies to choose from. This reduces the dependency on single strategies and provides stable robust solutions. \par
         
     \begin{figure}[t]
        \centering
        \includegraphics[width=0.48\textwidth, height=5 cm]{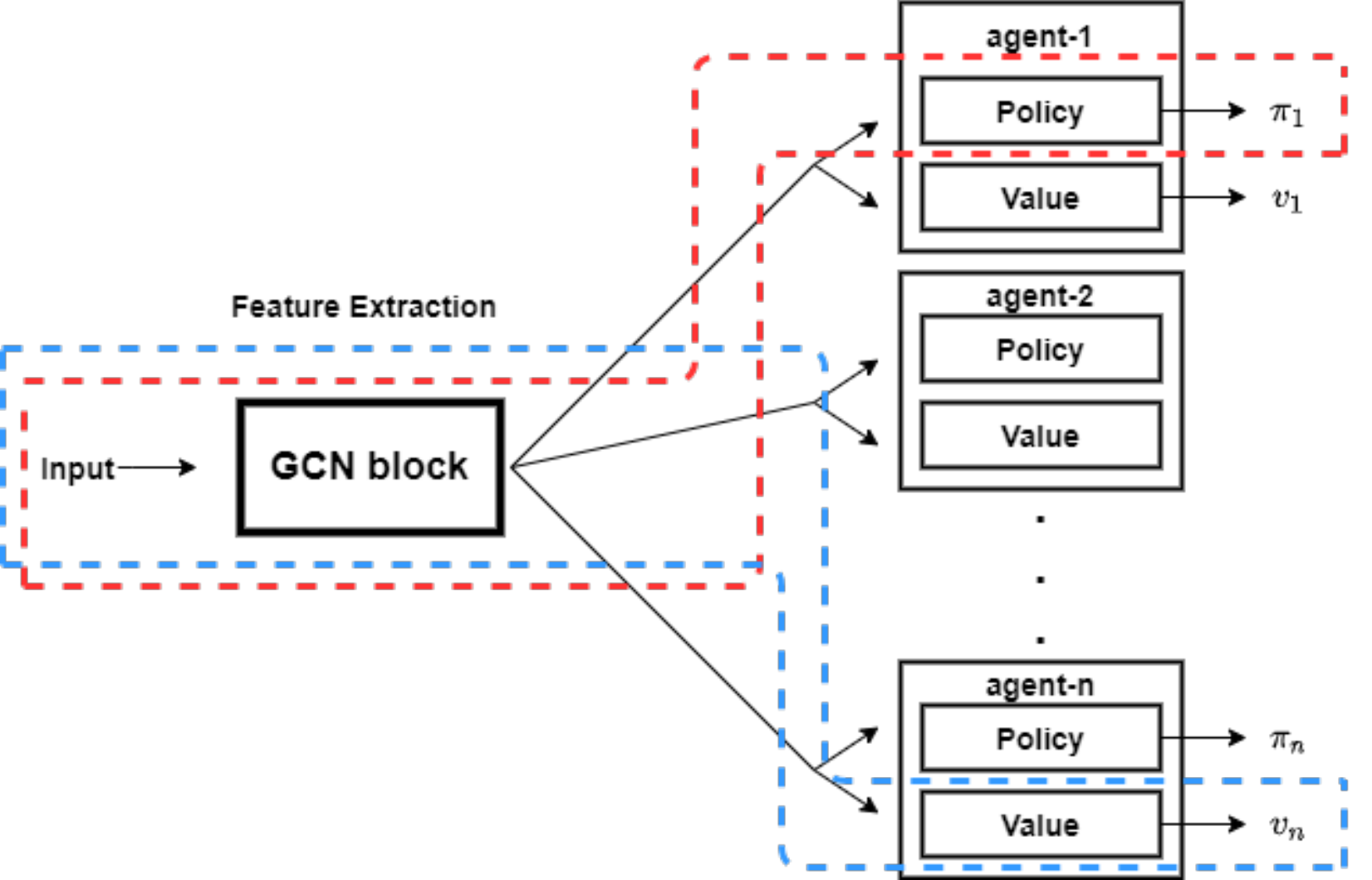}
        \caption{Neural network architecture of SiCLOP-NN}
        \label{fig:nn}
    \end{figure}

    \begin{algorithm}
    \caption{Training algorithm for SiCLOP-NN}
    \label{alg:train}
    \KwIn{set of agents $\mathcal{N}$, prediction neural network SiCLOP-P\textsubscript{NN} and data storage $d$}
    \KwOut{Updated prediction neural network SiCLOP-NN}
    \For{$epoch$ in $1:nEpochs$}
    {
        \For{$i$ in $1:|\mathcal{N}|$}
        {
            prepare batch $B_{i}$ of state-action and state-reward pairs from $d_{i}$ of recent episodes by random selection \\
            $loss \gets$ \textit{policy\_error + value\_error} (Eq. \ref{eq:policyloss} and \ref{eq:vloss})\\
            minimize $loss$
        }
    }
    \end{algorithm}
    
    \subsection{SiCLOP-NN, Prediction with Neural Network}
    
        The neural network SiCLOP-NN is a GCN based policy predictor and state value estimator. It has two parts, feature extraction and prediction, as shown in Figure \ref{fig:nn}. The first part is used for graph embedding as our algorithm focuses on CGs. This feature extraction part consists of multiple GCN layers. The number of GCN layers determine the number of hops of feature information among agents. Two GCN layers would mean that the features of the agents will reach other agents up to two hops away. The second part takes the extracted feature vectors of the agents from the GCN block and generates predicted best response policies and state value estimations for every agent. The state values are aggregated to construct a combined state value. In Figure \ref{fig:nn}, The red outlined pipeline is policy predictor and the blue outlined pipeline is state value estimator. SiCLOP-NN returns a boltzman probability distribution over the actions to enable weighted sampling of actions in SiCLOP-NN.\par
        
        SiCLOP-NN takes the local information of the agents and a CG as input. The setting of the local information can be defined, for example, the cells within a limited range of the agent. The CG is also constructed according to predefined rules. In SiCLOP, the input is created in the \textit{preprocess} function where the local information of the agents are collected into a set in lines 39-40. Then the set is returned with the constructed CG. \par
        
        The training of the neural network occurs after a fixed number of simulation episodes. In Algorithm \ref{alg:train}, a fixed-sized batch $B_{i}$ of \textit{state-policy} and \textit{state-value} pairs are randomly selected from the data from recent episodes for every agent. The reason behind selecting data from recent episodes is to train the neural network on recently found policies of the agents. This allows the agents to adapt to new policies faster. Then predictions are made on that batch of states to calculate the cross-entropy error between the real policy and predicted policy for agent $i$. This is calculated using
        \begin{equation}
            policy\_error = \sum_{(s^{b}, \pi^{b}_{i}) \in B_{i}} -\pi^{b} \cdot ln(SiCLOP-NN(s^{b},i))
            \label{eq:policyloss}
        \end{equation}
        And the mean squared error between real value and predicted value is calculated with
        \begin{equation}
            value\_error = \sum_{(s^{b}, v^{b}_{i}) \in B_{i}} ||v^{b}_{i} - SiCLOP-NN(s^{b},i)||^{2}
            \label{eq:vloss}
        \end{equation} 
        The neural network is updated to minimize the combined loss.\par
        
        The neural network trains and predicts on states consisting of any number of agents, with their internal interactions passed as CG into the GCN block. This adds the flexibility of performing predictions on variable number of agents. On the other hand, the input is a collection of observations of the agents dictated by a constraint of locality, i.e. communication range. SiCLOP-NN's architecture allows making prediction in environments of any size and shape. If the agents are trained in a smaller environment with fewer agents, the learned agents can still apply the learned policies on much larger environments with more agents as their policy prediction only depends on the locality of information.
        
\section{Empirical Results}

\begin{figure}
    \centering
    \includegraphics[width=0.45\textwidth]{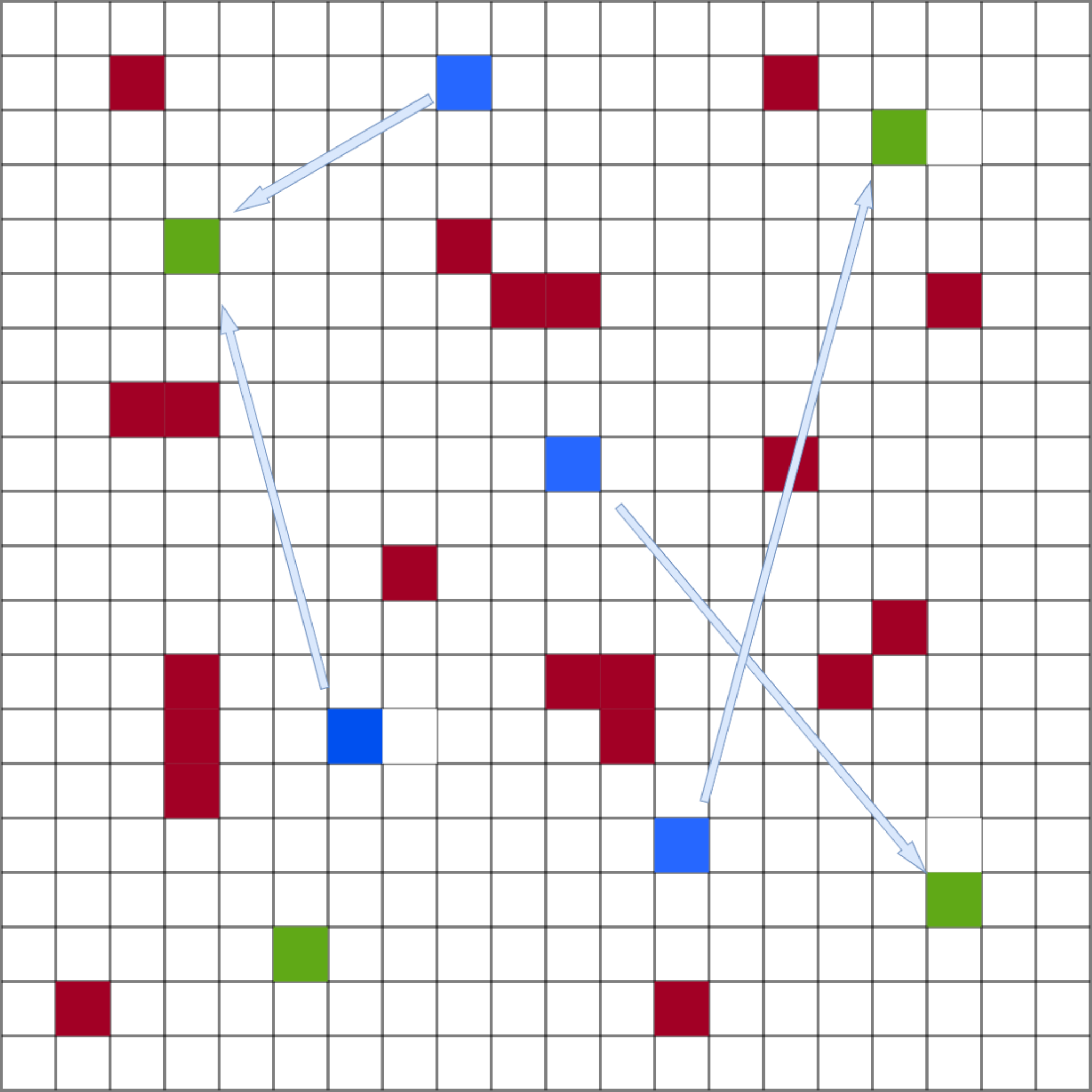}
    \caption{In this environment, \textit{green} cells are destinations, \textit{blue} cells are agents and \textit{red} cells are obstacles. The arrows from \textit{blue} cells to \textit{green} denote the destination of the agents.}
    \label{fig:example}
\end{figure}

To test and compare our algorithm, we used a modified version of the drone delivery environment \cite{choudhury2021scalable}. The modification is the inclusion of obstacles to make the environment more realistic, making the movement of the agents difficult. The environment is shown in Figure \ref{fig:example}. Here, agents need to reach their designated objectives but agents can collide with the obstacles or other agents if they try to enter the same cell (can be objective cell) at the same time. They can also choose actions that take them out of the environment area. These actions are heavily penalized. The goal of the environment design is to test if the agents can learn to avoid penalizing actions by using cooperation and complete the tasks, as seen in most realistic problems. Every agent gets $1.0$ point for reaching an objective, $-1.0$ for collision and $-0.5$ for going out of bounds. There are also penalties if multiple agents stick too close to each other. Agents are rewarded if they get closer to their objectives. The agents have a total of $9$ possible actions, going to the $8$ surrounding cells and staying in place. The MMDP episode ends when all agents reach their goals or the pre-defined number of steps are used up. The results are shown as average score per agent as the number of agents vary in multiple environmental setting. \par
Next, we will present an analysis of how the algorithm performs, how it can scale up and how it performs compared to FV-MCTS-MP and ABC planning and learning algorithms. 

    \subsection{Agent Training and Transfer Learning}
    The first experiment is to observe how the agents improve their policies to get better results. In this experiment, there are $8$ agents in a $15x15$ grid with $15$ obstacles. The agents have maximum 40 steps to complete the tasks. A total of 500 randomly generated problem states were generated and SiCLOP was tested on these 500 solution episodes. Agent policy predictor was trained every 10 episodes. For every step in the episode, the simulator can either run for a maximum amount of time or can expand limited number of simulation nodes. The algorithm was tested on three such settings. To test the scaling up ability of the algorithm, the trained policy predictors were used to test if it can generate good quality solutions on new environments of different sizes and difficulties.\par
    
    \begin{figure}
     \centering
     \includegraphics[width=0.45\textwidth]{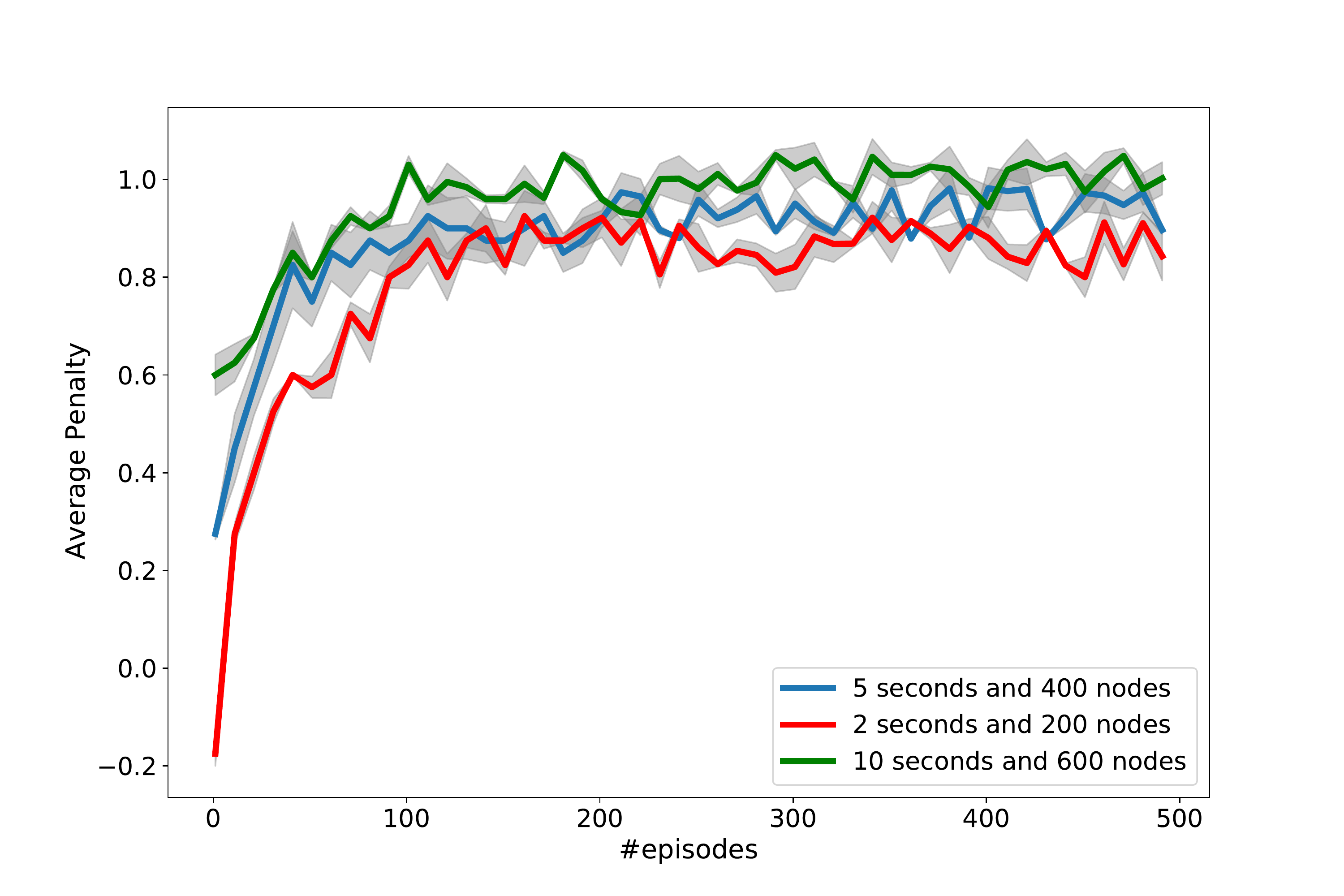}
     \caption{Evolution of the agents in different settings}
     \label{fig:ex1}
    \end{figure}
    
    \begin{table}[t!]
    \centering
    \begin{tabular}{|p{1.6cm}||p{1.6cm}|p{1.6cm}|p{1.6cm}|}
     \hline
     \multicolumn{4}{|c|}{Environment List} \\
     \hline
     Grid Size & Agents & Obstacles & Average Score\\
     \hline
     8x8   & 5 & 4 & 1.004\\
     20x20 & 20 & 10  & 1.06\\
     50x50 & 100 & 150 &  1.35\\
     100x100 & 150 & 200 &  1.52\\
     \hline
    \end{tabular}
    \caption{Transferring and testing the learned agents in new environments}
    \label{tab:transfer}
    \end{table}
    
     From the results shown in Figure \ref{fig:ex1}, we can see that at the beginning of the experiment, the agents adopt random policies and as the episodes progress they start to adapt to the rules of the environment and converge to a high score around $1.0$. An average score around $1.0$ denotes that the agents are avoiding penalties and succeeding at their given tasks. If the simulation limitations are tighter, the algorithm struggles to find good solutions at the beginning, even yielding to negative scores. But as SiCLOP is allowed to do more simulation before taking actions, it generates better solutions from the start, training the agents on quality solutions. This helps converging to optimal scores much faster, as we can see in Figure \ref{fig:ex1}. It is also shown that, in general, more simulation leads to better score on average. \par
     
     After training the agents in a $15x15$ grid, the policy predictor is transferred to randomly generated environments. These environments differ in size, number of agents and obstacles. In Table \ref{tab:transfer}, we can see that the average score per agent in these environments are above $1.0$, which implies that the algorithm is capable of adapting to a varying environment  and avoid penalties by using knowledge from a different environment. But in case of other planning algorithms, FV-MCTS-MP stores state values of visited states so it cannot use any prior knowledge on new environments and ABC uses the entire state to predict policies, so it can not be applied to an environment of different size. Transfer learning lets the algorithm to train on different environments, enabling agents to initially train in a small environment and then train in larger environments. As seen from the experiment, the prediction model can transfer the knowledge allowing the training in large environments to converge faster. This can save time and computational resources in case of complex environments and large single agent action spaces, making the algorithm applicable to most realistic MMDP problems.
     
     \begin{figure}[t]
     \centering
     \begin{subfigure}[b]{0.5\textwidth}
         \centering
         \includegraphics[width=\textwidth]{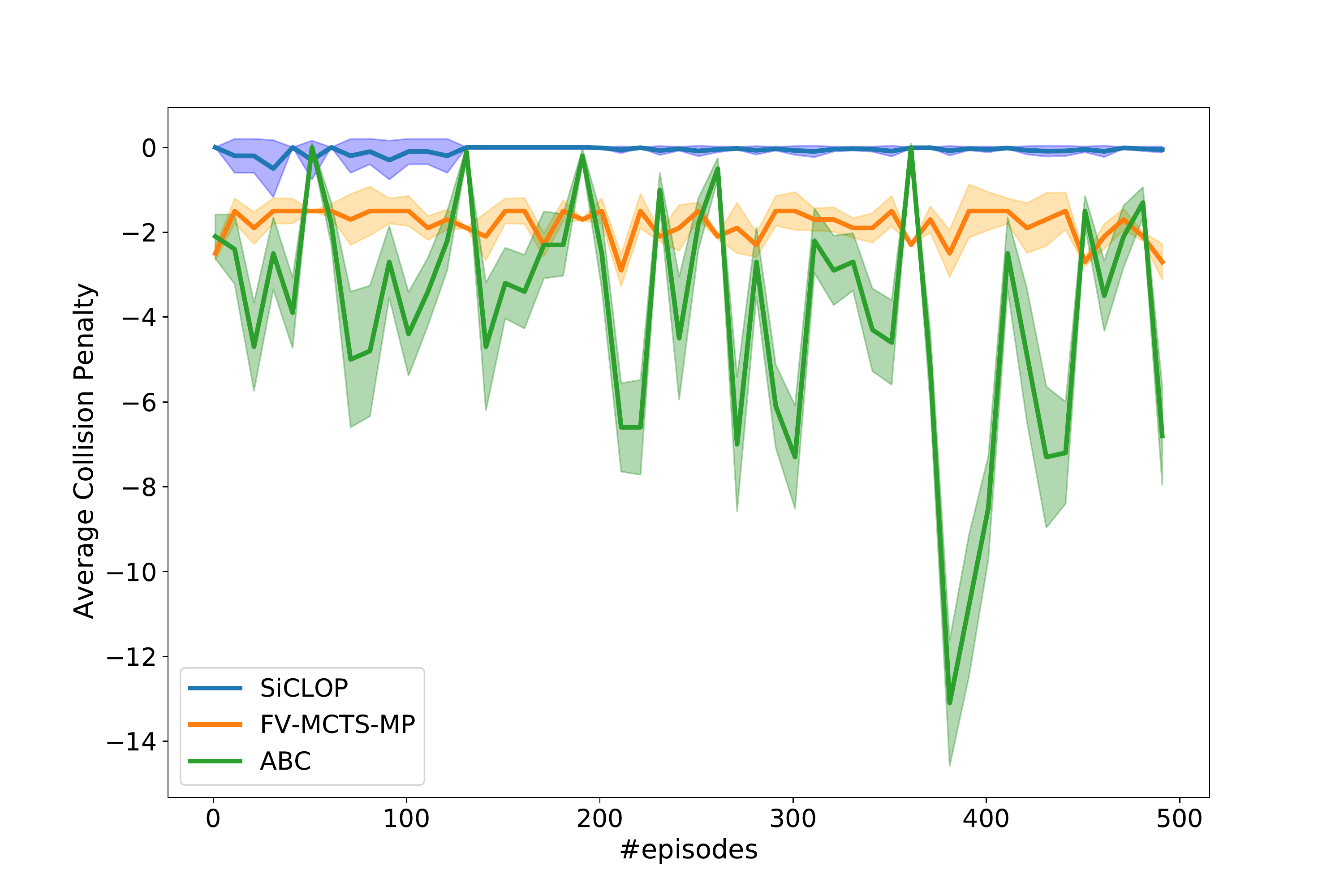}
     \end{subfigure}
     \hfill
     \begin{subfigure}[b]{0.5\textwidth}
         \centering
         \includegraphics[width=\textwidth]{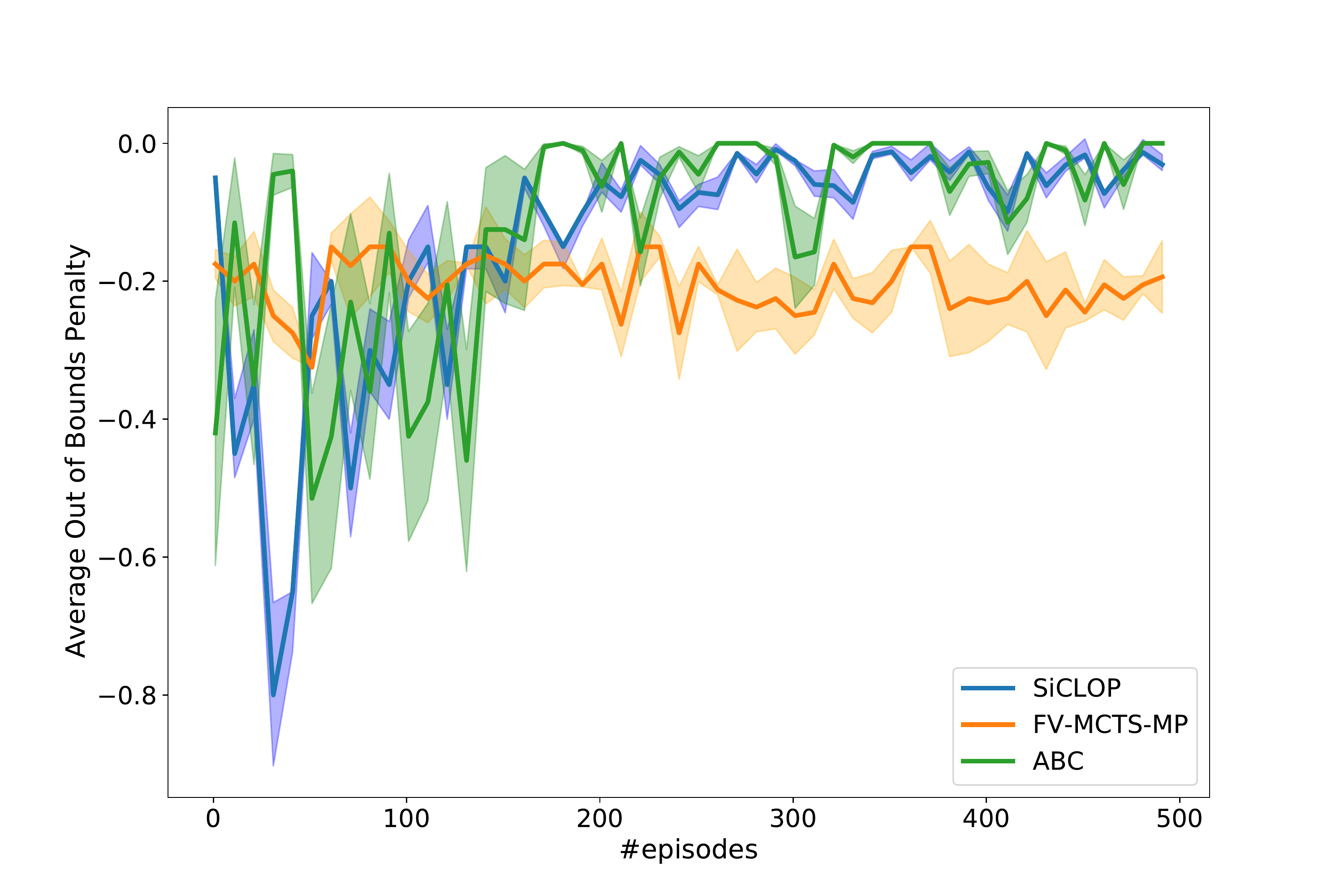}
     \end{subfigure}
     \caption{Comparison of adaptability to collision (top) and out of bounds (bottom)}
     \label{fig:ex2}
    \end{figure}
    
    \subsection{Learning Cooperation}
    In the second experiment, the cooperation learning and adaptability of the algorithms were tested. SiCLOP and ABC are both model learning algorithms and FV-MCTS-MP is a centralized tree searching algorithm. While ABC adopts a decentralized approach, SiCLOP chooses and learns its policies in a centralized manner. To test the cooperation learning ability of the agents, both algorithms were run on the same 500 episodes with $5$ seconds time limit per action in a $15x15$ grid with $5$ obstacles and $8$ agents. In Figure \ref{fig:ex2}, we can see how the algorithms learn cooperation and adapt to handle $collision$ and $out of bounds$ penalties.\par
    From Figure \ref{fig:ex2}, it is clear that ABC fails to learn how to avoid collisions but can learn to avoid going out of bounds. The penalty of going out of bounds depends on individual agents and collision mostly depends on the interaction of the agents. ABC learns the agent models one at a time and tries every available action. This helps the agent models to learn individual actions effectively. But in case of collision, it is very hard to avoid without communication among agents, causing the ABC agent models to not converge to effective policies. As there is no communication, there is no differentiation between successful and unsuccessful joint actions before the actions are executed.
    FV-MCTS-MP consistently outperforms ABC in cooperation but due to time limit and small number of iterations, there were no significant increase in performance. SiCLOP outperforms both algorithms as it searches over the joint policy space and needs lesser number of iterations to find cooperative joint policies to avoid the penalties. Throughout the training process, SiCLOP shows a stable performance and converges to almost no average penalty. \par
    
    Contrary to the collision penalty, all three algorithms manage to adapt to the out of bounds penalty. Before training, FV-MCTS-MP outperforms both SiCLOP and ABC but after training these two perform better. Similar to previous experiment, FV-MCTS-MP's performance does not get better because of the limitations. The reason behind the similar performances of SiCLOP and ABC is that avoiding out of bounds penalty does not require any cooperation.
    
    \subsection{Performance Comparison}
    In this experiment, SiCLOP was compared to FV-MCTS-MP and ABC by average scores accumulated on gradually increasing size of environments and number of agents. All of the algorithms had $10$ seconds time limit per action. SiCLOP and ABC algorithms were trained on 500 episodes. In Figure \ref{fig:ex3}, we can see the comparison in scores.\par
    \begin{figure}
        \centering
        \includegraphics[width=0.5\textwidth]{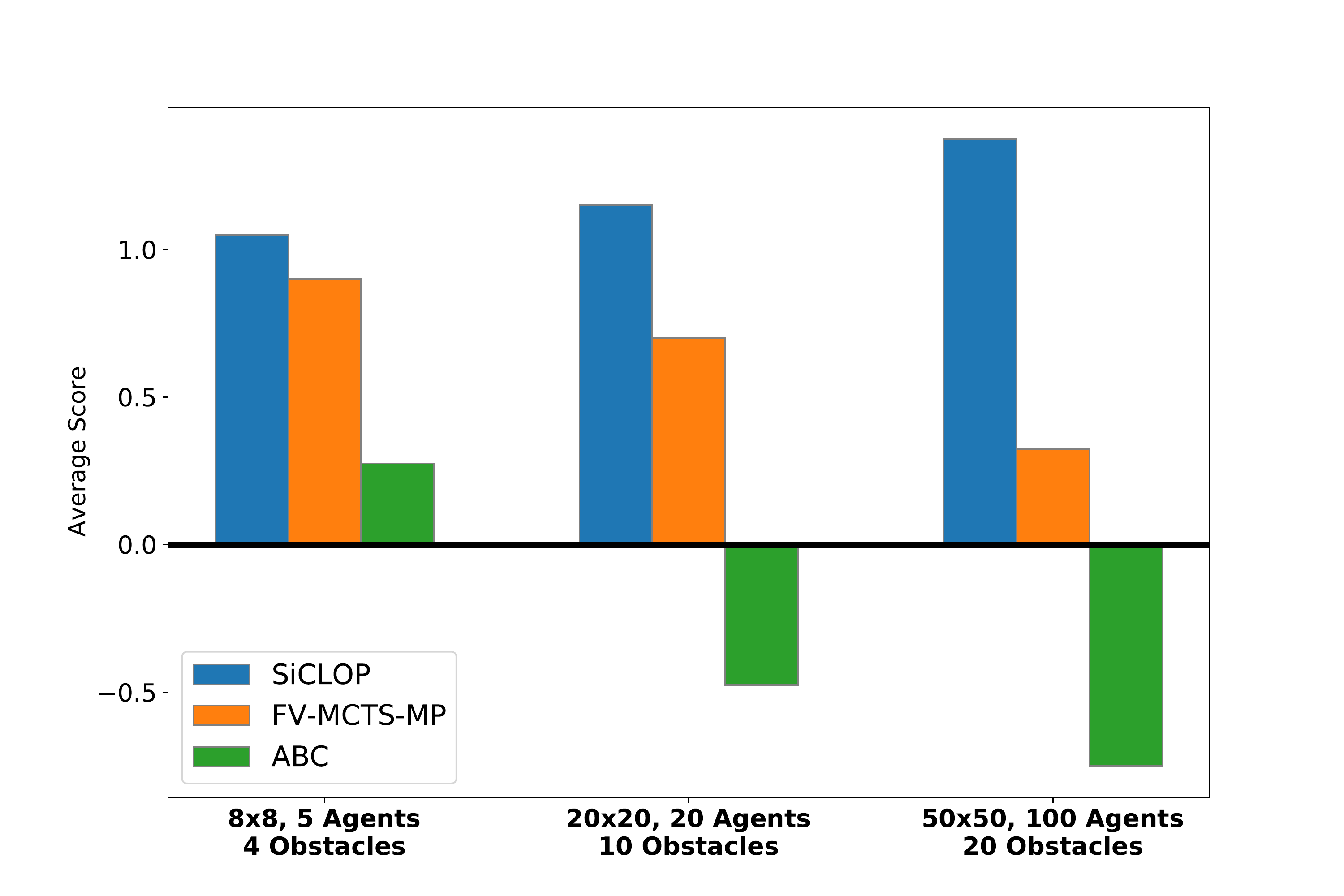}
        \caption{Performance comparison}
        \label{fig:ex3}
    \end{figure}
    The algorithms were tested in three different environments of varying sizes. As the size of the environment gets larger, accumulating a high score becomes harder as there are more obstacles and the objectives are farther. From this experiment, it is clear that as the environment gets larger, ABC is more prone to penalties and accumulates lower average scores. FV-MCTS-MP, with limited time and number of iterations, is unable to find sequence of actions that lead to high score. As the simulation tree of FV-MCTS-MP is wide, it calculates the expected action value over a wide and shallow tree, causing it to have average negative values. So, FV-MCTS-MP prefers not to move the agents to minimize penalties. But, as SiCLOP searches over a deeper search tree, it can find better solutions to the MMDP problems. Also, the size of the environment hardly affects the performance of SiCLOP as the agents only use their local information and interactions. These attributes and better cooperation learning allow SiCLOP to outperform both ABC and FV-MCTS-MP by a significant margin. 

\section{Conclusions and Future Work}
In this paper, we introduced a scalable and adaptive algorithm to solve MMDP. Our algorithm SiCLOP uses local information of the agents to construct state dependant dynamic CGs and uses it to find optimum policies through simulation and effective pruning. We have shown that our algorithm can adapt to larger, more realistic environments and outperform existing online MMDP solvers.\par
SiCLOP is an algorithm that uses policy predictions as a tool to find better policies, which implies that the quality of the initial prediction holds great significance on the final outcome. More work is required for finding better ways to train the prediction models, which may include newer ways of selecting self-play data on which the model is trained on and better neural network architectures that suit the process better.





\bibliographystyle{ACM-Reference-Format} 
\balance
\bibliography{sample}{}
\nocite{*}


\end{document}